\documentclass{article}
\usepackage[a4paper, total={6in, 8in}]{geometry}
\usepackage{times}

\usepackage{soul}
\usepackage{url}
\usepackage[hidelinks]{hyperref}
\hypersetup{
    colorlinks=true,
    linkcolor=blue,
    filecolor=magenta,      
    urlcolor=cyan,
    }
\usepackage[utf8]{inputenc}
\usepackage[small]{caption}
\usepackage{graphicx}
\usepackage{amsmath}
\usepackage{booktabs}
\usepackage{amsthm}
\usepackage{amsfonts} 
\usepackage[english]{babel}
\usepackage{color}
\usepackage[square,numbers]{natbib}
\usepackage{multirow}

\newtheorem{theorem}{Theorem}[section]

\newtheorem{lemma}[theorem]{Lemma}

\title{Deep Ordinal Regression using Optimal Transport Loss and Unimodal Output Probabilities}

\author{
Uri Shaham$^1$\footnote{Contact Author}
\and
Igal Zaidman$^2$
\and
Jonathan Svirsky\\
}
\date{%
    $^1$Center for Outcome Research and Evaluation, Yale University\\%
    $^2$Bar Ilan University\\[2ex]%
    \today
}

\begin{document}

\maketitle

\begin{abstract}
It is often desired that ordinal regression models yield unimodal predictions. 
However, in many recent works this characteristic is either absent, or implemented using soft targets, which do not guarantee unimodal outputs at inference. 
In addition, we argue that the standard maximum likelihood objective is not suitable for ordinal regression problems, and that optimal transport is better suited for this task, as it naturally captures the order of the classes.
In this work, we propose a framework for deep ordinal regression, based on unimodal output distribution and optimal transport loss.
Inspired by the well-known Proportional Odds model, we propose to modify its design by using an architectural mechanism which guarantees that the model output distribution will be unimodal.
We empirically analyze the different components of our proposed approach and demonstrate their contribution to the performance of the model.
Experimental results on eight real-world datasets demonstrate that our proposed approach consistently performs on par with and often better than several recently proposed deep learning approaches for deep ordinal regression with unimodal output probabilities, while having guarantee on the output unimodality. In addition, we demonstrate that proposed approach is less overconfident than current baselines.
\end{abstract}


\section{Introduction}\label{sec:intro}
Ordinal regression is an area of supervised machine learning, where the goal is to predict the value of a discrete dependent variable, whose set of (symbolic) possible values is ordered.
Despite often overshadowed by more common applications like classification and regression, ordinal regression covers a wide range of important applications, such as prediction of failure times, ranking, age estimation and many more.

Many practitioners often treat ordinal regression problems as classification or regression problems (for example, this was indeed the case with many submissions to Kaggle's Diabetic Retinopathy competition\footnote{https://www.kaggle.com/c/diabetic-retinopathy-detection/} in 2015). While having common characteristics with both classification and regression, ordinal regression can arguably be viewed a mid-point between the two.
An ordinal model is of course similar to a classification model, as both predict a discrete value (``label'') out of a finite set of possible ones. However, the existence of an order on set of labels, when available, can potentially lead to an improved performance, comparing to a standard classifier, which does not assume such order. This typically occurs via distinguishing between the severity of prediction mistakes: while in classification typically ''all mistakes are created equal", in ordinal regression different mistakes may be associated with different severity (for example, predicting "moderately-sized" when the ground truth value is "big" may be less severe than a "tiny" prediction.
In regression problems, the dependent variable naturally does take values from an ordered set. However, this set is typically a continuum. Moreover, regression performance may be sensitive to monotonic transformations of the dependent variable, while such sensitivity does not take place in ordinal regression problems. 
Hence one may expect that typical ordinal regression algorithms have potential to outperform classification or regression approaches, when the range of the dependent variable is finite and ordered. 
In section~\ref{sec:experiments} we will provide examples to the superiority of our proposed approach over classification and regression in benchmark tasks.

The arguably most fundamental ordinal regression model is the Proportional Odds Model (POM), a generalized linear model, similar in spirit to logistic regression, however where the logits are defined for cumulative probabilities. 
POM is typically trained via maximum likelihood (as is also the case for several recently proposed deep ordinal regression approaches, which will be reviewed in section~\ref{sec:related}).
We argue that likelihood is a sub-optimal measure of quality for ordinal regression setup, as it only considers the probability mass the model assigns to the true class, ignoring the remaining mass. This implicitly assumes that ``all mistakes are equal'', which, as discussed above, is not the case for ordinal regression. Hence we seek for an alternative measure of quality which may be more appropriate for ordinal regression. We argue that the optimal transport divergence might be a better fit. In addition, this divergence turns out to be particularly appealing, as it obtains a simple, differentiable form in the case that one of the distributions is Dirac, which is indeed the case in ordinal regression; this will be explained in Section~\ref{sec:preliminaries}. 

Another potential source of sub-optimality of POM (and of several recently-proposed approaches for deep ordinal regression) is the often-reasonable requirement that a probabilistic model for ordinal regression will output unimodal probabilities (i.e., that when moving in either direction from the most probable class, the probabilities predicted by the model will decay in a monotonic fashion. 
Although there are domains in which unimodality is not necessarily a desirable property, such as in movie rankings (where people may have either positive or negative definitive opinions about a particular movie), in many other real world domains it is a natural requirement, for example when predicting age of a person or a grade of a tumor, as it may be counter-intuitive to trust a model prediction which says that a predicted tumor grade is either 1 or 4, but not 2 or 3.
However, despite often being a desired characteristic, unimodality is unfortunately not always fulfilled. 
While this was identified by several recent works for deep ordinal regression, unimodality is often encouraged (but not enforced) via soft targets. 
In the next section, 
we will argue that this is a sub-optimal means to achieve unimodality. 
To the contrary, we propose  a novel mechanism to enforce unimodality of the output distribution, implemented via architectural design, and demonstrate that it does not hurt the level of performance.

Experimentally, we analyze the contribution of the optimal transport objective and our proposed unimodality mechanism to the performance of the model, and provide results on eight real world image benchmark datasets which demonstrate that our proposed approach consistently performs on par with and often better than several recently proposed approaches for deep ordinal regression, while having a unimodality guarantee. 
In addition, we demonstrate that the predictions made by our proposed approach tend to be consistently less overconfident than those of the competing methods, manifested in greater uncertainty in cases of wrong predictions.

\section{Related Work}\label{sec:related}

Being a traditional area of machine learning and statistics, there exists a large corpus of literature on ordinal regression. In this section we focus on approaches based on deep architectures.
Several such approaches were proposed in the recent years.
A common approach seems to be to turn the ordinal regression problem into a multi-label classification problem, for example~\cite{fu2018deep, liu2017deep, liu2018constrained, tv2019data, berg2020deep, cheng2008neural}.
We argue that the multi-label approach has two major problematic aspects: first, the output probabilities are not always guaranteed to be consistent, in the sense of increasing cumulative distribution (i.e., we would like to predict $\Pr(y \le 1) \le \Pr(y \le 2)\le\ldots\le \Pr(y \le k)$.
Second, even if the output probabilities are consistent, as is the case in~\cite{liu2018ordinal, cao2020rank} for example, the predicted class probabilities are not necessarily unimodal, i.e., there is no guarantee for existence of $j \in \{1,\ldots,k\}$ such that $\Pr(Y=1) \le \ldots \le \Pr(Y=j) \ge \ldots \ge \Pr(Y=k)$. This is the case in several recent works, e.g.,~\cite{liu2019probabilistic, vargas2020cumulative, pan2018mean}.

\cite{beckham2017unimodal} proposed an elegant mechanism to obtain unimodal output probabilities, based on either the Poisson or the Binomial distributions, which are both unimodal. In both cases their model outputs a scalar ($\lambda$ in case of the Poisson, $p$ in the case of the binomial), which is then mapped to a probability mass function that uses (after normalization) as the model output probabilities.
While being a convenient, architectural-based solution for the unimodality issue, their approach is inherently limited in its ability to express the level of uncertainty of the model's prediction: since a single parameter determines both the location of the mode, and the decay of the probabilities, the model cannot output a highly flat or highly peaked probability vector; in addition, instances of the same predicted class ought to have similar output probabilities.
Inspired by their approach, we utilize the normal distribution, in which one parameter determines the location while another determines the decay. In section~\ref{sec:experiments} we will demonstrate that this greater flexibility yields an improvement in performance. 
\cite{belharbi2019non} propose a constrained optimization approach to achieve unimodality. However, this comes at a cost of a somewhat cumbersome optimization process. More importantly, even if unimodality is indeed achieved for the train data, there are no guarantees that this will also be the case for unseen test data.

Several works propose to handle the unimodality requirement via soft targets, for example~\cite{gao2017deep, diaz2019soft, liu2019unimodal, liu2020unimodal}.
Despite the fact that usage of soft targets to obtain unimodality is sub-optimal, as it does not guarantee unimodal outputs at inference (and not even at train time), it often led to improved performance, comparing to~\cite{beckham2017unimodal}, where unimodality is guaranteed.  
In section~\ref{sec:experiments} we will demonstrate that the proposed approach enables one to enjoy both worlds, and have a unimodality guarantee while not hurting the quality of predictions.

Several works use cross entropy as a training objective, while using one-hot (or binary) targets, see, for example~\cite{belharbi2019non, vargas2020cumulative, fu2018deep, beckham2017unimodal, berg2020deep, tv2019data, cao2020rank}. As pointed out in several papers, and will also be demonstrated in section~\ref{sec:preliminaries}, in the case of one-hot targets, the cross entropy term equals the negative log of the probability assigned by the model to the true class, making it invariant to the distribution of the remaining probability mass. While a reasonable thing in a standard classification setting, this ignores the order of the classes, making it a sub-optimal choice for ordinal regression setting.
To overcome this limitation of cross entropy~\cite{hou2016squared}, followed by~\cite{beckham2017unimodal, liu2019unimodal} use optimal transport loss, which is a natural way to incorporate the order of the classes into the loss term. In this sense, it is similar to the approach we take in this manuscript. 

To summarize this section, we identify the following requirements for an appropriate ordinal regression model:
\begin{itemize}
    \item Unimodality of the model's output distribution.
    \item It is advantageous to enforce the unimodality via the design of the model, rather than via soft targets.
    \item A model utilizing one-hot targets should not be trained using cross entropy objective (or maximum likelihood in general).
    \item The decay of the output probabilities should reflect the uncertainty of the model in its predictions. 
\end{itemize}
These requirements naturally lead us to our proposed approach in section~\ref{sec:proposed}. However, before we specify it, we begin with a brief review of of the proportional odds model and optimal transport divergence.


\section{Preliminaries}\label{sec:preliminaries}
We begin this section with a description of the proportional odds model from a latent variable perspective. We then briefly review optimal transport as a divergence between two probability distributions.

\subsection{The Proportional Odds Model}\label{sec:pom}
Let $(X,Y)\in\mathcal{X}\times\mathcal{Y}$ be random variables, having joint probability $\mathcal{P}_{XY}$, where $\mathcal{X}=\mathbb{R}^d$ and $\mathcal{Y} = \{1,\ldots,k\}$, where $1,\ldots,k$ are considered as symbols. Let $\preceq$ be an order relation defined on $\mathcal{Y}$ such that $1 \preceq\ldots\preceq k$.
The proportional odds model is parametrized by $\alpha \in \mathbb{R}^{k-1}, \beta \in \mathbb{R}^d$ and applies to data $\{(x_i, y_i)\}_{i=1}^n$, sampled i.i.d from $\mathcal{P}_{XY}$. 

Let $\epsilon$ be a logistic random variable (thus having a sigmoid cumulative distribution function $F(x)=\frac{1}{1 + \exp(-x)}$), and let Z be a random variable defined as $Z = \beta^T X + \epsilon$. 
The entries of $\alpha$ use to define the cumulative conditional probabilities via
\begin{equation}
    \Pr (Y \preceq j | X=x) = \Pr(Z\le \alpha_j) = F(\alpha_j - \beta^T x). \label{eq:pom}
\end{equation}
Similarly to logistic regression, this yields linear log-odds (logits), however, defined with respect to cumulative terms
\begin{equation}
    \gamma_j 	\equiv \log\frac{\Pr (Y \preceq j | X=x)}{\Pr (Y \succ j | X=x)} = \alpha_j - \beta^Tx.\notag
\end{equation}

It is convenient to interpret equation~\eqref{eq:pom} by viewing $\beta^Tx$ as a factor that shifts the standard logistic density function, while the $\alpha_j$ terms are thresholds, with respect to which the cumulative probabilities are defined. This is depicted in Figure~\ref{fig:pom}.
\begin{figure}[t]
  \centering
    \includegraphics[height=.35\textwidth]{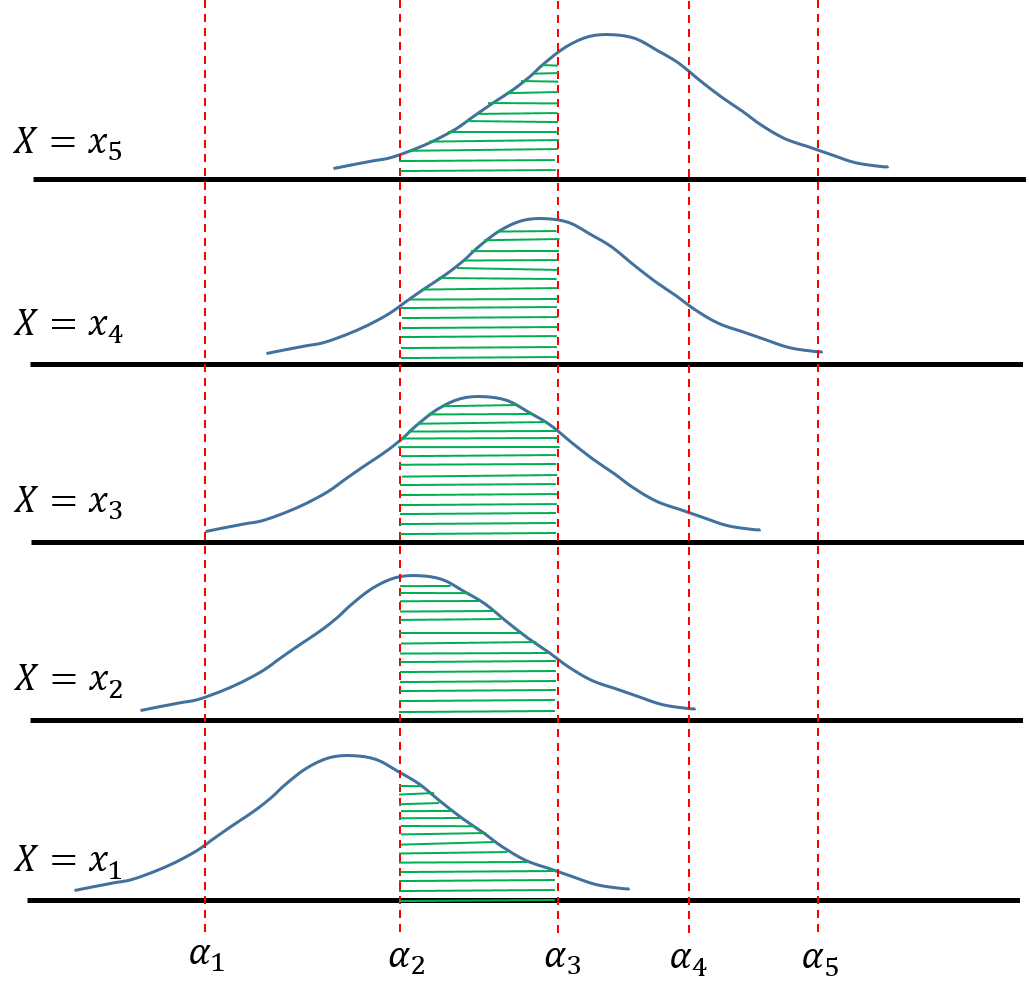}  
    \caption{The proportional odds model. $x_i$ is a realization of $X$. The standard logistic density is shifted by $\beta^T x_i$. The thresholds $\alpha_j$ define the bins which determine the probability predicted by the model to each class. For example, the green area defines the probability $\Pr(Y=3)$.} 
    \label{fig:pom}
\end{figure}
Let $(x, y)$ be a realization of $(X,Y)$. The likelihood assigned by the model to $(x,y)$ is
\begin{align}
    L(\alpha, \beta ; (x, y)) =&\Pr(Y=y | X=x; \alpha, \beta) \notag\\
    =&{\textit F}(\alpha_y - \beta^Tx) - {\textit F}(\alpha_{y-1} - \beta^Tx), \label{eq:likelihood}
\end{align}
considering $\alpha_0=-\infty$ and $\alpha_k=\infty$.  
The model is typically trained in a standard fashion by maximizing the log-likelihood function on the training data.

Despite its popularity, the POM suffers from two main issues:
First, the model's output probabilities are not necessarily unimodal. This is depicted in Figure~\ref{fig:not_unimodal}.
\begin{figure}[t]
  \centering
    \includegraphics[height=.12\textwidth]{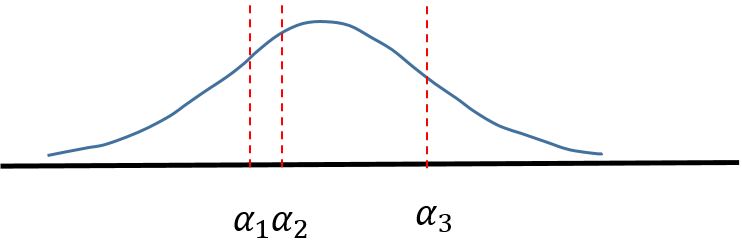}  
    \caption{POM does not always output unimodal probabilities. 
    Recall that $\Pr (Y=k) = \Pr(Z\le \alpha_k) - \Pr(Z\le \alpha_{k-1})$. Therefore, the above plot shows an example where the output probabilities are such that $\Pr(Y=1) > \Pr(Y=2) < \Pr(Y=3) > \Pr(Y=4)$ (i.e., the output probabilities are bimodal and not unimodal).} 
    \label{fig:not_unimodal}
\end{figure}
Second, the likelihood function~\eqref{eq:likelihood} depends only on the probability the model assigns to the correct class $y$, and is invariant to the way the remaining probability mass is assigned by the model. This ignores the order on the label set, and hence does not use important information that might be used to improve prediction quality, as depicted in Figure~\ref{fig:likelihood}.
\begin{figure}[t]
  \centering
    \includegraphics[height=.20\textwidth]{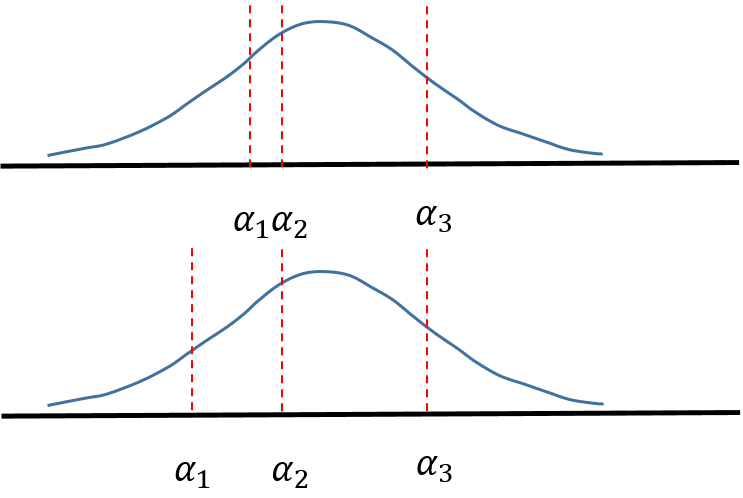}  
    \caption{The likelihood function of POM is invariant to the way the predicted probability mass of the incorrect classes is assigned. In the above example the correct class is 3, and the two instances have the same likelihood, despite the fact that in the bottom case, the probability mass assigns to neighboring class 2 is larger, making the bottom case more appropriate than the top one.} 
    \label{fig:likelihood}
\end{figure}

It is important to mention that as cross entropy term is essentially equivalent to model's negative log-likelihood function, this invariance to the partition of the remaining mass over the incorrect classes is common to all models trained via cross entropy minimization, as long as the target labels are one-hot. 

In section~\ref{sec:proposed} we will show how our method overcomes these two limitations of the POM. 

\subsection{Optimal Transport}\label{sec:opt_tran}

Let $(M, d)$ be a finite metric space, and let $p, q$ be probability mass functions defined on M.
Optimal transport, also denoted as the 1-Wasserstein distance and the Earth Mover Distance, between $p$ and $q$ is
	\begin{equation}
    OT(p, q) = \inf_{\gamma\in\Gamma}\int_{M\times M}c(x,y)d\gamma(x,y),\label{eq:emd}
\end{equation}
where $\Gamma$ is the set of all joint probabilities on $M\times M$, having marginals $p$ and $q$, and the metric $c$ specifies the costs of moving probability mass between every two elements of $M$.
This amounts to the optimal transportation of probability mass that transforms $p$ into $q$ and vice versa.
In the general case, the distance can be found by solving a linear program, and several relaxations have been proposed to accelerate its computations while preserving its geometrical properties, see, for example,~\cite{peyre2019computational, feydy2019interpolating, cuturi2013sinkhorn}.
However, in the case where $p$ is a Dirac point mass (i.e., having a a one-hot probability mass function),  solving equation~\eqref{eq:emd} becomes trivial and becomes
\begin{equation}
    OT(p, q) = \sum_{i=1}^k q_i c(i,j),\label{eq:ot}
\end{equation}
where $j$ is the correct class, and $k$ is the total number of classes,
as is also depicted in Figure~\ref{fig:emd}.
\begin{figure}[t]
  \centering
    \includegraphics[height=.25\textwidth]{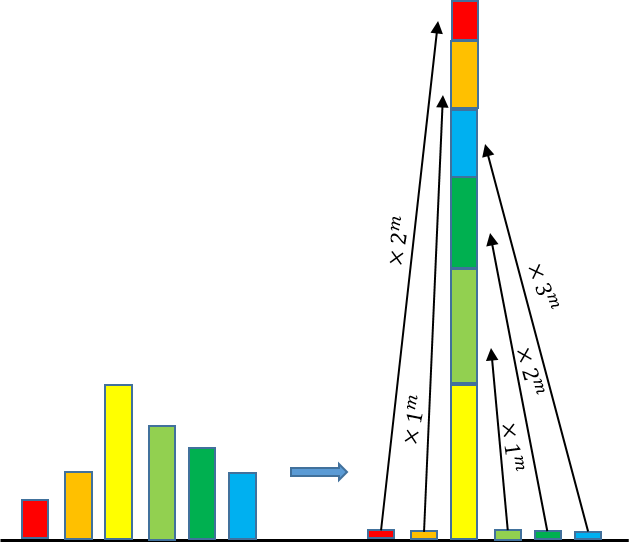}  
    \caption{Optimal transport between model output probability mass function and a Dirac (one-hot) probability mass, using $c(i, j) = |i-j|^m$ cost.} 
    \label{fig:emd}
\end{figure}
Letting $q$ denote a model's output probabilities and $p$ denote a one-hot target, equation~\eqref{eq:ot} is of course differential with respect to the model outputs $q$ and therefore can be used as a loss term for gradient-based optimization. 

The cost metric $c$ can incorporate domain knowledge in order to quantify the semantic distance between every two elements of $M$.
Since in our case $M=\{1,\ldots, k\}$ is an ordered space, a natural possibility is to define $$c(i, j) = |i-j|^m,$$ for some $m \ge 1$, i.e., ,mapping the symbolic class labels to consecutive integers, and computing powers of absolute differences.
When $m=1$, the optimal transport can also be computed as the $\ell_1$ distance between the cumulative mass functions\footnote{This holds when the classes are ordered.}, $\|\text{CMF}(p) - \text{CMF}(q)\|_1$, (see~\cite{levina2001earth}, for example). This is equivalent to the computation in equation~\eqref{eq:ot}, and also generalizes to arbitrary targets (i.e., not necessarily Dirac).

\section{The Proposed Approach}\label{sec:proposed}

In this section we describe our proposed mechanism for architectural-based generation of unimodal output probability distributions.

\subsection{Rational}
The fact that the unimodality is obtained directly via architectural design has a major advantage over using soft targets for training, since the output probabilities are guaranteed to be unimodal for every input instance, as is also the case for the mechanism proposed by~\cite{beckham2017unimodal}. 
However, unlike~\cite{beckham2017unimodal}, our proposed approach employs the normal distribution, depending separately on a location parameter and a scale parameter, so that the location of the mode is detached from the decay of the probability mass,  which yields a more flexible design than the single-parameter distributions used by ~\cite{beckham2017unimodal}, in which the single parameter determines both the mode and the decay. We will demonstrate in section~\ref{sec:experiments} that this greater flexibility is helpful in expressing prediction uncertainty.

\subsection{Unimodal Output Probabilities Generation}
Inspired by the POM, we utilize thresholds to define bins, so that the total mass inside each bin is the output probability of the corresponding class.
However, observe that the lack of unimodality of POM can be fixed by letting the bins be of equal length and remain fixed during training. 

Therefore, instead of learning the thresholds, during training a map $x\mapsto(\mu, \sigma)$ is learned, where $\mu$ is a location parameter, and $\sigma$ is a scale parameter, which define a $\mathcal{N}(\mu, \sigma^2)$ distribution, using which the output probabilities are computed. 

Formally, we divide the range $[-1, 1]$ to $k$ equal bins, where $k$ is the number of classes, defined by $-1=\alpha_0,\alpha_1,...,\alpha_k=1$,
so that $\alpha_{i} - \alpha_{i-1} = \frac{2}{k}$.
The (un-normalized) probabilities are given by
\begin{equation}
    \tilde{p}_i(x) = \Phi_{\mu(x),\sigma(x)}(\alpha_i) -  \Phi_{\mu(x),\sigma(x)}(\alpha_{i-1}),\label{eq:unnorm_probs}
\end{equation}
where $\Phi_{\mu,\sigma}(\cdot)$ is the $\mathcal{N}(\mu, \sigma^2)$ cumulative distribution function, and we have emphasized that $\mu,\sigma$ are in  fact functions of the input instance $x$. Since the bins cover [-1,1] and not the entire real line, we normalize the probabilities to obtain proper model predictions via
\begin{equation}
    \Pr{(Y=i | x)} =  p_i(x) \equiv \frac{ \tilde{p}_i(x)}{\sum_{j=1}^k  \tilde{p}_j(x)}.\label{eq:outputs}
\end{equation}
To compensate for the fact that the probability generating mechanism depends on less parameters than POM (2 for the former, $d + k-1$ for the latter), the map $x\mapsto(\mu, \sigma)$ is expressed via a deep network, which is therefore able to represent a complex nonlinear relation.
Our proposed mechanism for generation of unimodal output probabilities is depicted in Figure~\ref{fig:proposed}.
\begin{figure}[t]
  \centering
    \includegraphics[height=.10\textwidth]{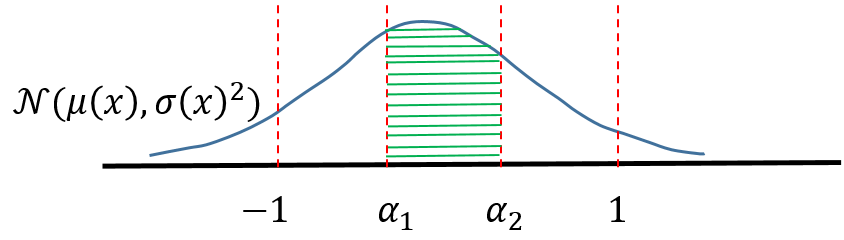}  
    \caption{Generation of unimodal output probabilities for $k=3$ classes. An input $x$ is mapped to a $(\mu, \sigma)$ pair, which define a normal distribution $\mathcal{N}(\mu(x), \sigma(x)^2)$ over the real line. The output probabilities are proportional to the mass in the bins, which are of equal length. The green area equals to the un-normalized probability $\tilde{p}_2(x)$, corresponding to $\Pr(Y=2|X=x)$.} 
    \label{fig:proposed}
\end{figure}

The following lemma, proved in Appendix~\ref{app:lemma}, establishes that the model output probabilities are indeed unimodal.
\begin{lemma}\label{lemma}
Let $x\in \mathbb{R}^d$ be an input to the model, which is mapped to $\mu=\mu(x), \sigma=\sigma(x)$, and let $p_1,\ldots p_k$ be the model output probabilities, generated via equation~\eqref{eq:outputs}. Then $p_1,\ldots p_k$ define a unimodal multinomial random variable.
\end{lemma}

We remark that all arguments made here with regard to the normal distribution also hold for other unimodal distributions, which are symmetric around $\mu$, such as Logistic($\mu, \sigma$) and Cauchy($\mu, \sigma$), which both have slower decay patterns.

\subsubsection{Training procedure}
To summarize this section, given an instance $x$, our model outputs parameters $\mu_\theta(x), \sigma_\theta(x)$ of a normal distribution, where $\theta$ collectively denotes the weights of the network. $\mu_\theta(x), \sigma_\theta(x)$ determine the un-normalized categorical probabilities $\tilde{p}_i(x),\; i=1,\ldots,k$ via equation\eqref{eq:unnorm_probs}, which are then normalized to  $p_i(x),\; i=1,\ldots,k$ using equation~\eqref{eq:outputs}. 
Following equation~\eqref{eq:ot}, the objective equation used to train the model is 
$$ l(\theta; x) = \sum_{i=1}^k p_i(x) c(i,j),$$ where $j\in{1,\ldots,k}$ is the correct class of $x$. 
In our experiments we used $c(i,j) = |i-j|$.


\section{Experimental Results}\label{sec:experiments}
To analyze the components of the proposed approach, we begin this section with small-scale ablation studies on the Abalone dataset. We then report experimental results on eight real world benchmark image datasets.

\subsection{Ablation Studies}
Here we use a modified version of the Abalone dataset\footnote{\url{https://archive.ics.uci.edu/ml/datasets/abalone}}. 
The dependent variable counts the number of rings inside the abalone shell, corresponding to its age, which should be predicted from 8 numeric features.
We slightly re-arranged the range of the dependent variable using a hand-crafted monotonic transform, resulting in a a partition of the data to eight classes, with number of instances per class ranging between 391 and 689.
The data file appears in supplementary material.

Table~\ref{tab:abalone_results} shows the performance of our proposed approach, along with the following baselines:
\begin{itemize}
    \item A neural network regression model, trained using squared error loss. The predicted class is obtained by rounding the model's (linear) output.
    \item A neural network classification model, trained using cross entropy loss and one-hot targets.
    \item A standard POM. As POM is a generalized linear model, we expect that it may be outperformed by the other baselines, which are all non-linear.
\end{itemize}
In addition, to evaluate the contribution of the optimal transport objective and the unimodal probabilities mechanism, we also train two ``hybrid'' models:
\begin{itemize}
    \item A neural network classification model, trained using optimal transport loss and one-hot targets (classification OT)
    \item A neural network model, with our proposed mechanism for unimodal output probabilities, trained using cross entropy loss, and one-hot targets (Unimodal CE)
    
\end{itemize}
All neural network models shared the same architecture, except for the output layer, and were trained using identical batch sizes and learning rate policies.

\begin{table}[ht]
\centering
\scalebox{0.9}{
\begin{tabular}{||c |c |c |c ||}
 \hline
 Method & MAE & QWK & Spearman\\ [0.5ex] 
 \hline\hline
 Regression & 1.01 $\pm$ .02 & 0.75 $\pm$ .02 & 0.76 $\pm$ .02 \\\hline 
 Classification & 1.12 $\pm$ .02 & 0.72 $\pm$ .02 & 0.73 $\pm$ .01 \\ \hline
 POM & 1.43 $\pm$ .01 & 0.63 $\pm$ .01 & 0.73 $\pm$ .01 \\\hline
 Proposed & {\bf0.98 $\pm$ .03} & {\bf 0.78 $\pm$ .01} & {\bf 0.79 $\pm$ .02}\\\hline
 \hline
 Classification OT & 1.12 $\pm$ .02 & 0.69 $\pm$ .01 & 0.75 $\pm$ .01 \\\hline
 Unimodal CE & 1.70 $\pm$ .54 & 0.63 $\pm$ .11 & 0.67 $\pm$ .07\\\hline
\hline

\end{tabular}}
\caption{Performance of various methods on the Abalone dataset, in terms of mean $\pm$ standard deviation over 5 independent trials.}
\label{tab:abalone_results}
\end{table}
Table~\ref{tab:abalone_results} demonstrates a few interesting properties of the proposed approach. First, by analyzing the performance of the baseline models, we observe that the proposed approach outperforms standard regression and classification, in terns of both Mean Absolute Error (MAE), Spearman correlation and Quadratic Weighted Kappa (QWK). 
All three methods (the proposed, classification and regression) outperform POM on this dataset, perhaps due to the latter being a linear model, unlike the other ones.

By looking at the performance of the hybrid models, we see that  optimal transport objective, missing in Unimodal CE, and the proposed unimodal probabilities generation mechanism, missing in Classifier OT, both contribute to the performance. 
In addition, we observe that the optimal transport objective does not manifest its full advantage with softmax-generated probabilities, perhaps as softmax often yield a peaked probability mass function. Similarly, cross entropy seem to work better with standard softmax, comparing to our proposed unimodality mechanism.
We observed these phenomena also in several other experiments, which are not described here.

\noindent \textbf{Analysis of uncertainty}
~\cite{belharbi2019non} claims that methods that learn a variance parameter may simply push it to zero, hence this parameter need to be manually set.
We show that in our case, the learned map $x\mapsto (\mu,\sigma)$ does not yield a vanishing standard deviation parameter. Rather, the standard deviation corresponds to the uncertainty of the model in its predictions.
To see that, we plot in Figure~\ref{fig:histo} the histograms of the probability of the mode of the predicted distribution (i.e., histograms of $\max\{p_1,\ldots,p_k\}$), as a function of whether or not the mode predicts the ground truth class or not.
\begin{figure}[t]
  \centering
    \includegraphics[height=.22\textwidth]{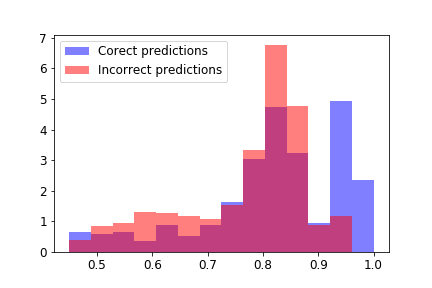}  
    \caption{Histograms showing the probability of the mode of the output probabilities for Abalone test instances with correct and incorrect predictions}. 
    \label{fig:histo}
\end{figure}
As can be seen, the modes are higher in cases of correct predictions, than in cases of incorrect predictions. Put another way, the prediction confidence is lower in wrong predictions, comparing to correct predictions. The higher the mode is, the more peaked the output probability function is, and the predictions indeed tend to be more accurate. 
We will observe a similar phenomenon in our experiments on real world datasets, reported next.

\subsection{Results on Real World Datasets}

\begin{table*}[h!]
\centering
\scalebox{0.99}{
\begin{tabular}{||l |l |c |c |c |c |c |c ||} 
 \hline
 Dataset & Method & MAE  & OOA & Spearman & QWK &\% Unimodal & Entropy ratio \\ [0.5ex] 
 \hline\hline
 \multirow{5}{*}{HCI}     & \cite{beckham2017unimodal} 
                          & .62 $\pm$ .04 
                          & .85 $\pm$ .02 
                          &  \textbf{.71 $\pm$ .03}
                          & .75 $\pm$ .02 
                          & \textbf{1  $\pm$ 0} 
                          & .82 $\pm$ .1\\\cline{2-8}
                          & \cite{liu2019unimodal}     
                          & .57 $\pm$ .05 
                          & .86 $\pm$ .02 
                          & .68 $\pm$ .04 
                          & .74 $\pm$ .03 
                          & .4  $\pm$ .04 
                          & .81 $\pm$ .1\\\cline{2-8}
                          & \cite{gao2017deep}         
                          & .71 $\pm$ .04 
                          & .87 $\pm$ .02     
                          & .67 $\pm$ .03     
                          & .7  $\pm$ .03     
                          & .99 $\pm$ .02     
                          & \textbf{1.28 $\pm$ .13}\\\cline{2-8}
                          & \cite{diaz2019soft}        
                          & .57 $\pm$ .04 
                          & .86 $\pm$ .02  
                          & .69 $\pm$ .03  
                          & .75 $\pm$ .02 
                          & .96 $\pm$ .01 
                          & .71 $\pm$ .06\\\cline{2-8}
                          & Proposed                    
                          & \textbf{.54  $\pm$ .03} 
                          & \textbf{.89  $\pm$ .01} 
                          & \textbf.7  $\pm$ .03
                          & \textbf{.77  $\pm$ .02} 
                          & \textbf{1    $\pm$ 0}       
                          & .95 $\pm$ 0.13\\\cline{2-8}\hline\hline
 \multirow{5}{*}{Adience} & \cite{beckham2017unimodal} 
                          & .53 $\pm$ .08 
                          & .94 $\pm$ .02 
                          & \textbf{.9 $\pm$ .02} 
                          & .91 $\pm$ .03 
                          & \textbf{1    $\pm$ 0}             
                          & 1 $\pm$ .01\\\cline{2-8}
                          & \cite{liu2019unimodal}     
                          & .48 $\pm$ .06
                          & .94 $\pm$ .02 
                          & .88 $\pm$ .03 
                          & .9  $\pm$ .03 
                          & .53 $\pm$ .05    
                          & 1.08 $\pm$ .01\\\cline{2-8}
                          & \cite{gao2017deep}         
                          & .5  $\pm$ .08 
                          & .94 $\pm$ .02 
                          & .88 $\pm$ .02 
                          & .9  $\pm$ .03 
                          & .6  $\pm$ .05 
                          & 1.  $\pm$ .01\\\cline{2-8}
                          & \cite{diaz2019soft}        
                          & .47 $\pm$ .07 
                          & .94 $\pm$ .02 
                          & .89 $\pm$ .01 
                          & .91 $\pm$ .03 
                          & .99 $\pm$ .01       
                          & 1.06 $\pm$ .03\\\cline{2-8}
                          & Proposed                    
                          & .\textbf{45 $\pm$ .05} 
                          & \textbf{.95 $\pm$ .01} 
                          & \textbf{.9  $\pm$ .02} 
                          & \textbf{.92 $\pm$ .02} 
                          & \textbf{1   $\pm$ 0}       
                          & \textbf{1.19 $\pm$ .16} \\\cline{2-8}\hline\hline 
\multirow{5}{*}{Retina MNIST}       
                          & \cite{beckham2017unimodal} 
                          & .78  $\pm$ .02 
                          & .8   $\pm$ .01     
                          & .6   $\pm$ .01     
                          & .55  $\pm$ .02     
                          & \textbf{1    $\pm$ 0}    
                          & 1.02 $\pm$ 0.011\\\cline{2-8}
                          & \cite{liu2019unimodal}     
                          & .69 $\pm$ .02 
                          & .82 $\pm$ .01     
                          & .6  $\pm$ .02     
                          & .58 $\pm$ .02     
                          & .69 $\pm$ .04    
                          & 1.05  $\pm$ .02\\\cline{2-8}
                          & \cite{gao2017deep}         
                          & .76 $\pm$ .09 
                          & .8  $\pm$ .04     
                          & .59 $\pm$ .06     
                          & .54 $\pm$ .09     
                          & .94 $\pm$ .07       
                          & 1.  $\pm$ .01\\\cline{2-8}
                          & \cite{diaz2019soft}        
                          & .77 $\pm$ .06 
                          & .79 $\pm$ .3     
                          & .57 $\pm$ .05     
                          & .56 $\pm$ .05     
                          & .88 $\pm$ .03     
                          & 1.13 $\pm$ .01\\\cline{2-8}
                          & Proposed                    
                          & \textbf{.68 $\pm$ .01} 
                          & \textbf{.83  $\pm$ .01}     
                          & \textbf{.62  $\pm$ .01}     
                          & \textbf{.6   $\pm$ .01}     
                          & \textbf{1    $\pm$ 0}     
                          & \textbf{2.08 $\pm$ .08}\\\cline{2-8}
\hline\hline
 \multirow{5}{*}{FG-NET}  
                          & \cite{beckham2017unimodal} 
                          & .46 $\pm$ .01 
                          & .94 $\pm$ .03 
                          & .75 $\pm$ .04
                          & .8  $\pm$ .03 
                          & \textbf{1  $\pm$ 0} 
                          & 1.01 $\pm$ .01\\\cline{2-8}
                          & \cite{liu2019unimodal}     
                          & .36 $\pm$ .05 
                          & .96 $\pm$ .01 
                          & .82 $\pm$ .06 
                          & .83 $\pm$ .05 
                          & .2  $\pm$ .04 
                          & 1.27 $\pm$ .1\\\cline{2-8}
                          & \cite{gao2017deep}         
                          & .46 $\pm$ .05 
                          & .94 $\pm$ .02     
                          & .75 $\pm$ .05     
                          & .77 $\pm$ .04     
                          & .09 $\pm$ .03     
                          & 1. $\pm$ .1\\\cline{2-8}
                          & \cite{diaz2019soft}        
                          & .38 $\pm$ .05 
                          & .95 $\pm$ .02  
                          & .8  $\pm$ .04  
                          & .83 $\pm$ .04 
                          & .98 $\pm$ .01 
                          & 1.13 $\pm$ .04\\\cline{2-8}
                          & Proposed                    
                          & \textbf{.35  $\pm$ .03} 
                          
                          & \textbf{.98  $\pm$ .01} 
                          & \textbf{.84  $\pm$ .02} 
                          & \textbf{.87  $\pm$ .03} 
                          & \textbf{1    $\pm$ 0}       
                          & \textbf{1.58 $\pm$ 0.16}\\\cline{2-8}
\hline\hline
 \multirow{5}{*}{AAF}     
                          & \cite{beckham2017unimodal} 
                          & .44 $\pm$ .01 
                          & .97 $\pm$ .01 
                          & \textbf{.83 $\pm$ .01}
                          & .85 $\pm$ .05 
                          & \textbf{1  $\pm$ 0} 
                          & 1. $\pm$ .01\\\cline{2-8}
                          & \cite{liu2019unimodal}     
                          & .39 $\pm$ .01 
                          & \textbf{.98 $\pm$ .01} 
                          & .82 $\pm$ .01 
                          & .85 $\pm$ .01 
                          & .7  $\pm$ .03 
                          & 1.08 $\pm$ .01\\\cline{2-8}
                          & \cite{gao2017deep}         
                          & .4 $\pm$ .01 
                          & .98 $\pm$ .01     
                          & .82 $\pm$ .01     
                          & .85  $\pm$ .01 
                          & .9  $\pm$ .03     
                          & 1. $\pm$ .01\\\cline{2-8}
                          & \cite{diaz2019soft}        
                          & .39 $\pm$ .01 
                          & \textbf{.98 $\pm$ .01}  
                          & .82 $\pm$ .02  
                          & .85 $\pm$ .01
                          & \textbf{1 $\pm$ .01} 
                          & 1.05 $\pm$ .01\\\cline{2-8}
                          & Proposed                    
                          & \textbf{.38  $\pm$ .01} 
                          & \textbf{.98  $\pm$ .01} 
                          & \textbf{.83  $\pm$ .01} 
                          & \textbf{.86  $\pm$ .01} 
                          & \textbf{1    $\pm$ 0}       
                          & \textbf{1.2 $\pm$ 0.01}\\\cline{2-8}
\hline\hline
 \multirow{5}{*}{AFAD-LITE}     
                          & \cite{beckham2017unimodal} 
                          & .51 $\pm$ .01 
                          & .91 $\pm$ .01 
                          & .67 $\pm$ .01
                          & .69 $\pm$ .01 
                          & \textbf{1  $\pm$ 0} 
                          & 2.18 $\pm$ .05\\\cline{2-8}
                          & \cite{liu2019unimodal}     
                          & .5 $\pm$ .01 
                          & .92 $\pm$ .01 
                          & .67 $\pm$ .01 
                          & .69 $\pm$ .01 
                          & .96  $\pm$ .01 
                          & 1.11 $\pm$ .01\\\cline{2-8}
                          & \cite{gao2017deep}         
                          & .5 $\pm$ .01 
                          & \textbf{.93 $\pm$ .01}     
                          & .67 $\pm$ .01     
                          & .69 $\pm$ .01     
                          & \textbf{1 $\pm$ .01}     
                          & 1 $\pm$ .01\\\cline{2-8}
                          & \cite{diaz2019soft}        
                          & .5 $\pm$ .01 
                          & .92 $\pm$ .01  
                          & .67 $\pm$ .01  
                          & .69 $\pm$ .01 
                          & \textbf{1 $\pm$ .01} 
                          & 1.1 $\pm$ .01\\\cline{2-8}
                          & Proposed                    
                          & \textbf{.49  $\pm$ .01} 
                          & \textbf{.93  $\pm$ .01} 
                          & \textbf{.68  $\pm$ .01} 
                          & \textbf{.7  $\pm$ .01} 
                          & \textbf{1    $\pm$ 0}       
                          & \textbf{2.26 $\pm$ 0.23}\\\cline{2-8}
\hline\hline
 \multirow{5}{*}{EVA}     
                          & \cite{beckham2017unimodal} 
                          & .63 $\pm$ .02 
                          & .92 $\pm$ .01 
                          & \textbf{.6 $\pm$ .03}
                          & \textbf{.6 $\pm$ .02} 
                          & \textbf{1  $\pm$ 0} 
                          & 1. $\pm$ .01\\\cline{2-8}
                          & \cite{liu2019unimodal}     
                          & .61 $\pm$ .02 
                          & .92 $\pm$ .01 
                          & .56 $\pm$ .03 
                          & .56 $\pm$ .03 
                          & .68  $\pm$ .02 
                          & 1.04 $\pm$ .01\\\cline{2-8}
                          & \cite{gao2017deep}         
                          & .62 $\pm$ .03 
                          & .92 $\pm$ .01     
                          & .55 $\pm$ .03     
                          & .54  $\pm$ .03     
                          & .91 $\pm$ .02     
                          & 1. $\pm$ .01\\\cline{2-8}
                          & \cite{diaz2019soft}        
                          & .59 $\pm$ .03 
                          & .93 $\pm$ .01  
                          & .57 $\pm$ .03  
                          & .55 $\pm$ .03 
                          & .99 $\pm$ .01 
                          & 1.01 $\pm$ .01\\\cline{2-8}
                          & Proposed                    
                          & \textbf{.58  $\pm$ .02} 
                          & \textbf{.94  $\pm$ .01} 
                          & .58  $\pm$ .03
                          & .56  $\pm$ .03 
                          & \textbf{1    $\pm$ 0}       
                          & \textbf{1.05 $\pm$ 0.02}\\\cline{2-8}
\hline\hline
 \multirow{5}{*}{WIKI}     
                          & \cite{beckham2017unimodal} 
                          & .68 $\pm$ .01 
                          & .92 $\pm$ .01 
                          & .68 $\pm$ .01
                          & .68 $\pm$ .01 
                          & \textbf{1  $\pm$ 0} 
                          & 1. $\pm$ .01\\\cline{2-8}
                          & \cite{liu2019unimodal}     
                          & \textbf{.42 $\pm$ .01} 
                          & \textbf{95. $\pm$ .01} 
                          & \textbf{.69 $\pm$ .01} 
                          & .7 $\pm$ .01 
                          & .95  $\pm$ .01 
                          & \textbf{1.05 $\pm$ .01}\\\cline{2-8}
                          & \cite{gao2017deep}         
                          & .44 $\pm$ .01 
                          & .94 $\pm$ .01     
                          & .68 $\pm$ .01     
                          & .7  $\pm$ .01     
                          & .97 $\pm$ .01     
                          & 1. $\pm$ .01\\\cline{2-8}
                          & \cite{diaz2019soft}        
                          & .44 $\pm$ .01 
                          & .94 $\pm$ .01  
                          & \textbf{.69 $\pm$ .01}  
                          & .7 $\pm$ .01 
                          & .99 $\pm$ .01 
                          & 1.03 $\pm$ .01\\\cline{2-8}
                          & Proposed                    
                          & .43  $\pm$ .01 
                          & \textbf{.95  $\pm$ .01} 
                          & \textbf{.69  $\pm$ .01} 
                          & \textbf{.71  $\pm$ .01} 
                          & \textbf{1    $\pm$ 0}       
                          & 1.01 $\pm$ 0.01\\\cline{2-8}\hline\hline

\end{tabular}}

\caption{Performance of various methods on real world datasets, in a mean $\pm$ std format.}
\label{tab:real_results}
\end{table*}

\subsubsection{Datasets}
We evaluate our method on eight real world benchmark image datasets, involving various ordinal regression tasks: age-detection (Adience~\cite{eidinger2014age}, FG-Net~\cite{fu2014interestingness}, AAF~\cite{cheng2019exploiting}, AFAD-LITE~\cite{niu2016ordinal}, WIKI~\cite{Rothe-IJCV-2018}), bio-medical image classification (Retina-MNIST~\cite{yang2021medmnist}), historical image dating (HCI~\cite{palermo2012dating}) and image aesthetics estimation (EVA~\cite{kang2020eva}). A more detailed description of the datasets appears in Appendix~\ref{app:datasets}. Some examples from the Adience, HCI and Retina-MNIST datasets are shown in Figures~\ref{fig:adience},~\ref{fig:hci} and~\ref{fig:retina}.

\begin{figure}[t]
  \centering
    \includegraphics[height=.22\textwidth]{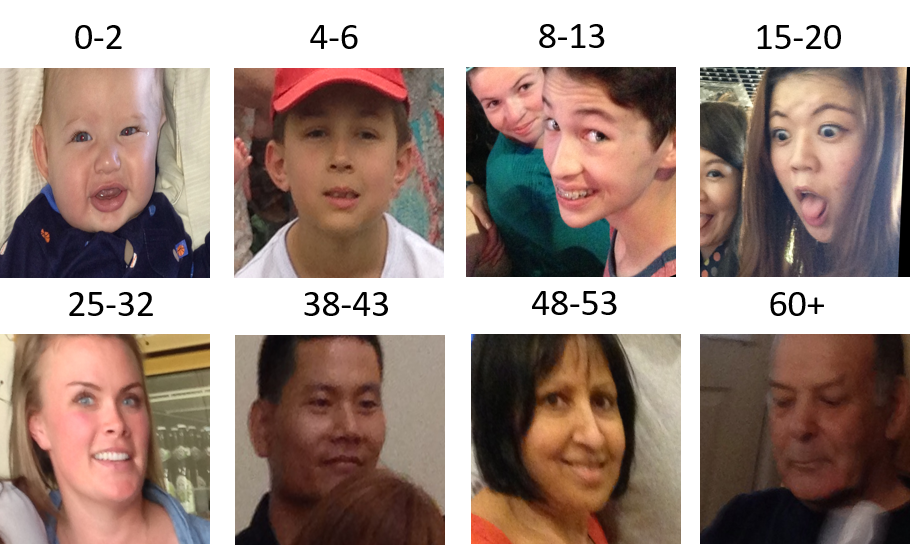}  
    \caption{Examples from the Adience dataset. Age category is indicated above each image.}. 
    \label{fig:adience}
\end{figure}

\begin{figure}[t]
  \centering
    \includegraphics[height=.16\textwidth]{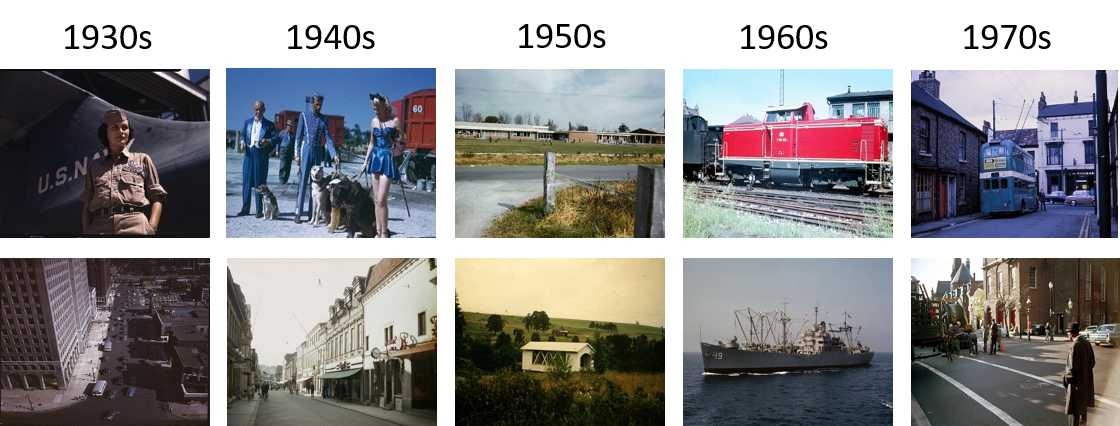}  
    \caption{Examples from the HCI dataset. Decades categories are indicated above.}. 
    \label{fig:hci}
\end{figure}

\begin{figure}[t]
  \centering
    \includegraphics[height=.109\textwidth]{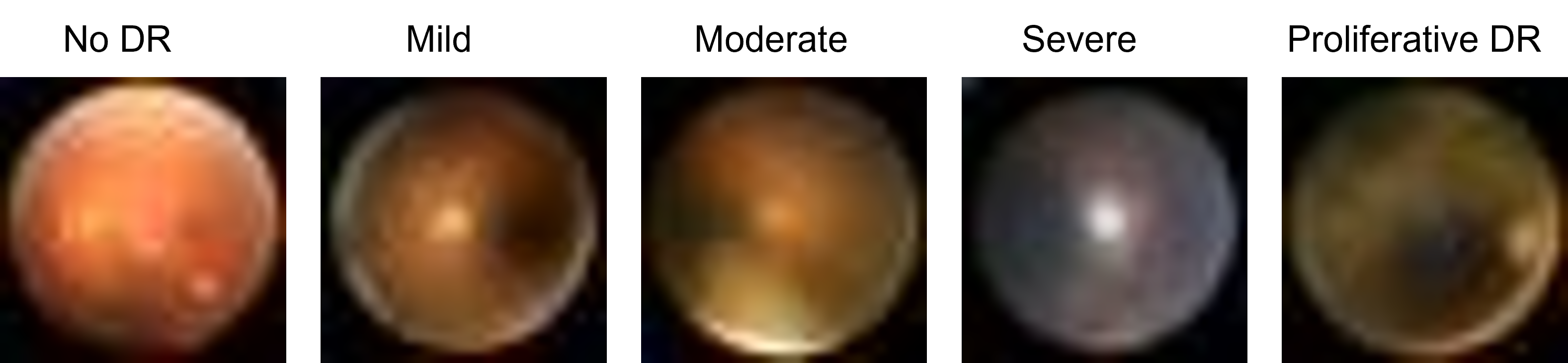}  
    \caption{Examples from the Retina mnist dataset. Diabetic Retinopathy classes are indicated above.} 
    \label{fig:retina}
\end{figure}

\subsubsection{Methods Compared}
We compare our proposed approach to four recently proposed approaches for deep ordinal regression, with unimodal output probabilities:
\begin{itemize}
    \item DLDL~\cite{gao2017deep}, an approach utilizing soft labels, generated using squared exponentially decaying distributions, trained using Kullback-Leibler divergence minimization (equivalent to cross entropy minimization).
    \item SORD~\cite{diaz2019soft}, an approach utilizing soft labels, generated using linear exponentially decaying distributions, trained using Kullback-Leibler divergence minimization.
    \item \cite{beckham2017unimodal}, where architectural-based unimodal output probabilities are generated using binomial distribution (single-learned parameter), trained using optimal transport loss.
    \item \cite{liu2019unimodal}, an approach utilizing soft labels, created as a mixture of Dirac, uniform and linear exponentially decaying distributions, trained using optimal transport loss.
\end{itemize}
In order to perform a fair comparison, we implemented all methods, using the same image transformations, backbone CNN and training procedures, so that the methods differ only in their output layer architectures and loss functions. We performed 5 independent trials, using the same train-validation-test splits for all methods. Additional technical details can be found in Section~\ref{seq:details}. For reproducibility, our GitHub repository~\url{https://github.com/jsvir/uniord} contains code reproducing the results reported in this section. 

\subsubsection{Evaluation Metrics}
We report several commonly-used evaluation metrics for ordinal regression tasts: MAE, One-Off Accuracy (OOA), Spearman correlation, QWK, as well as the percentage of test examples with unimodal predicted output probabilities. 
In addition, viewing the output probabilities as a multinomial distribution, we evaluate the level of overconfidence by comparing the ratio of the average entropy in incorrect predictions, to the average entropy in correct predictions.
A higher ratio is therefore desirable and indicates a lower level of overconfidence.

\subsubsection{Results}

Table~\ref{tab:real_results} shows the test results of each method on the eight benchmark datasets.
As can be seen, the proposed approach performs at least on-par and often better than the compared baselines, in a fairly consistent manner, across the various datasets and evaluation metrics. 
In addition, observe that only the proposed approach and the method of \cite{beckham2017unimodal} output unimodal probabilities, both via architectural design, while the other baselines, trained using soft targets, do not always output unimodal probabilities.
Moreover, on six of the eight datasets, the proposed approach outperforms all baselines also in terms of overconfidence.

\section{Conclusion}\label{sec:Conclusion}
In this manuscript we presented an approach for deep ordinal regression, inspired by the proportional odds model, utilizing an architectural mechanism for generation of unimodal output probabilities, and trained using optimal transport objective. 
We empirically analyzed the components of the proposed approach, and demonstrated that they both contribute to the performance of the model. 
We demonstrated that while performing on-par with or better than other recently proposed approaches for ordinal regression, our proposed method enjoys the benefits of guaranteed unimodal output probabilities, and of less overconfidence about its incorrect predictions.



{\small
\bibliographystyle{ieee_fullname}
\bibliography{ordinal}
}
\clearpage
\begin{appendix}

\section{Proof of Lemma~\ref{lemma}}\label{app:lemma}

\begin{proof}
Let $p_i, p_{i+1}$ be the output probabilities of two adjacent classes, and let $-1=\alpha_0,\ldots, \alpha_k=1$ be the thresholds. 
We will show that (i) if $\mu \le \alpha_{i-1}$ then $p_i \ge p_{i+1}$. Symmetrical argument will then imply that if $\mu \ge \alpha_{i+1}$ then $p_{i+1} \ge p_{i}$ (ii) if $\mu \in (\alpha_{i-1}, \alpha_{i})$ then $p_i > p_{i+1}$, whenever the latter exists. Similarly, this would imply that $p_i > p_{i-1}$. Together, (i) and (ii) will imply the statement of the lemma.

Denote by $f$ the density of the $\mathcal{N}(\mu, \sigma^2)$ distribution. To prove (i) observe that $\tilde{p}_i > \frac{2f(\alpha_{i})}{k} > \tilde{p}_{i+1}$, which implies $p_i > p_{i+1}$.

To prove (ii), divide the $i$'th bin to two sub-bins $B_{i, 1}, B_{i, 2}$, of lengths $a=\mu-\alpha_{i-1}$ and $b=\alpha_i-\mu$, respectively. Similarly, divide the $i+1$'th bin to two bins $B_{i+1, 1}, B_{i+1, 2}$ lengths $b$ and $a$.
Then from (i) 
\begin{equation}
    \int_{B_{i, 2}}f(x)dx > \int_{B_{i+1, 1}}f(x)dx. \label{eq:b1}
\end{equation}
In addition, observe that
\begin{align}
\int_{B_{i, 1}}f(x)dx &= \int_0^af(\mu + x)dx\notag\\ 
&> \int_0^af(\mu + 2b + x)dx\notag\\
&=\int_{B_{i+1, 2}}f(x)dx.\label{eq:b2}
\end{align}
Adding up equations~\eqref{eq:b1} and~\eqref{eq:b2}, we obtain $\tilde{p}_i > \tilde{p}_{i+1}$, which gives $p_i > p_{i+1}$.

\end{proof}

\section{Datasets}\label{app:datasets}
Tabel~\ref{tab:datasets} contains information on the benchmark datasets used for our experiments. 

 \noindent \textbf{HCI (Historical Color Image)} dataset contains 1326 images, partitioned to 5 classes, corresponding to decades from 1930s to 1970s, and the task is to associate each images with the decade it was taken at.
Random affine, radom horizontal / vertical flips and random crops of 224 transformations are applied during the training. The images are normalized in each color channel with mean and standard deviation of 0.5. The dataset was randomly split to the train/test as described in the Table~\ref{tab:technical}.

\noindent \textbf{Adience}: During the training the images are resized to (256,256) and random crop of size 224 and random horizontal flip are applied as augmentations. 

\noindent \textbf{FG-Net}: We partitioned the dataset to 8 classes, corresponding to decades. Augmentations were same as in the Adience experiment.

\noindent \textbf{RetinaMNIST} dataset has  5 classes and we apply random affine, horizontal and vertical flips as augmentations during the training. The size of the images is (28,28) as it provided by the dataset contributors. The train/test splits are proved by the contributors and were used as-is.

\noindent \textbf{AFAD-LITE} is a subset of the full AFAD dataset, which contains images of 22 continuous ages (for 15 to 40), in the amount of 60K.
We partitioned the dataset into 4 classes and augmentations were the same as in the Adience experiment.

\noindent \textbf{EVA (Explainable Visual Aesthetics)} dataset containes 4070 images aesthetically ranked from 0 to 10 by multiple voters. We calculate the average score for each image and partition the data into 5 classes in accordance with the average score. Augmentations were same as in the Adience experiment.

\noindent \textbf{AAF (All-Age-Faces)} dataset is already pre-processed and contains 13,322 face images (mostly Asian), distributed across all ages (from 2 to 80). We partitioned the dataset into 6 classes and augmentations were the same as in the Adience experiment.

\noindent \textbf{WIKI} dataset contains 62,328 images of celebrities from wikipedia with ages ranging form 1 to 100. Images with wrong timestamps are removed, and the dataset is partitioned into 6 classed. Augmentations were same as in the Adience experiment.

Figures~\ref{fig:afad},~\ref{fig:fgnet},~\ref{fig:eva},~\ref{fig:aaf},~\ref{fig:wiki} show examples from the AFAD-LITE, FG-Net, EVA, AAF and WIKI datasets.

\begin{figure}[t]
  \centering
    \includegraphics[height=.12\textwidth]{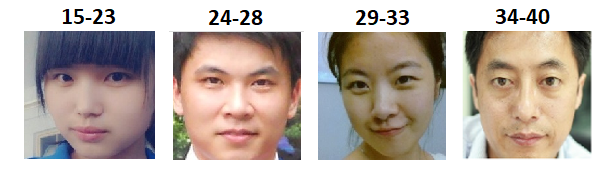}  
    \caption{Examples from the AFAD-LITE dataset. Age classes are indicated above. } 
    \label{fig:afad}
\end{figure}

\begin{figure}[t]
  \centering
    \includegraphics[height=.25\textwidth]{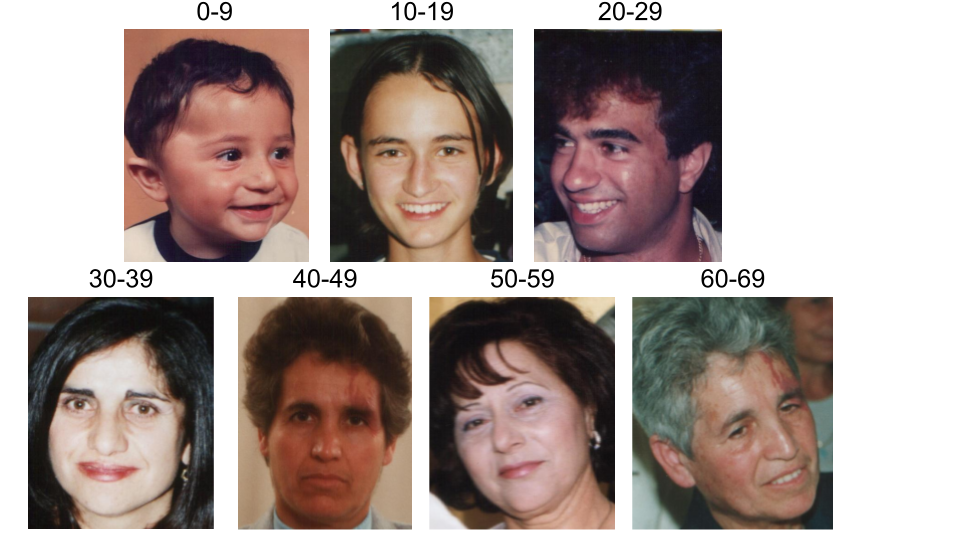}  
    \caption{Examples from the FG-Net dataset. Age classes are indicated above. 
    \label{fig:fgnet}}
\end{figure}

\begin{figure}[t]
  \centering
    \includegraphics[height=.1\textwidth]{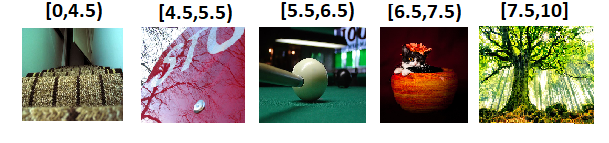}  
    \caption{Examples from the EVA dataset. Aesthetics classes are indicated above. }
    \label{fig:eva}
\end{figure}

\begin{figure}[t]
  \centering
    \includegraphics[height=.10\textwidth]{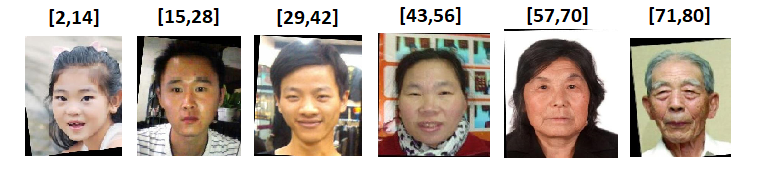}  
    \caption{Examples from the AAF dataset. Age classes are indicated above. } 
    \label{fig:aaf}
\end{figure}

\begin{figure}[t]
  \centering
    \includegraphics[height=.1\textwidth]{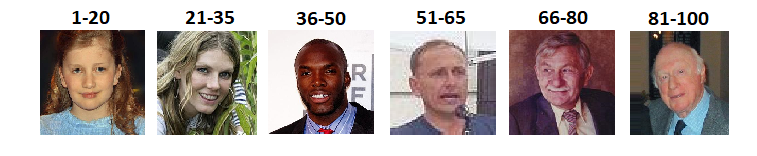}  
    \caption{Examples from the WIKI dataset. Age classes are indicated above. } 
    \label{fig:wiki}
\end{figure}

\begin{table*}[h!]
\centering
\scalebox{0.8}{
\begin{tabular}{||l |l | c | c  | c||}
 \hline
Dataset   & Task & Train images & Test Images & Classes \\ [0.5ex] \hline\hline
Adience      & age estimation         &\multicolumn{2}{c|} {pre-defined splits}   & 8 \\\hline 
HCI          & image dating           & 1,276 & 50 & 5\\\hline 
FG-Net       & age estimation         & 902 & 100 & 8\\\hline 
RetinaMNIST & DR classification      & 1200 & 400 & 5\\\hline
AFAD-LITE    & age estimation         & 37980 & 11869 & 4\\\hline
AAF          & age estimation         & 9058 & 2665 & 6\\\hline
EVA          & aesthetics estimation  & 2940 & 611 & 5\\\hline
WIKI         & age estimation         & 38660 & 12082 & 6\\\hline
 
\end{tabular}}
\caption{Benchmark datasets characteristics}
\label{tab:datasets}
\end{table*}

\section{Technical details}\label{seq:details}
Table~\ref{tab:technical} shows the technical details for the experiments on the real world benchmark datasets reported in this manuscript.

\begin{table*}[h!]
\centering
\scalebox{0.8}{
\begin{tabular}{||l |c | c |c  |c |c ||}
 \hline
Dataset   & Epochs & Batch size & initial LR & Decay LR after (epochs) & Weight Decay \\ [0.5ex] \hline\hline
Adience      & 100 & 64 & $10^{-4}$ & 40 & $10^{-5}$ \\\hline 
HCI          & 500 & 16 & $10^{-4}$ & 100, 300 & $10^{-3}$\\\hline 
FG-Net       & 100 & 64 & $10^{-4}$ & 40 & $10^{-4}$\\\hline 
RetinaMNIST  & 100 & 16 & $10^{-4}$ & 80, 90 & $10^{-4}$\\\hline
AFAD-LITE    & 100 & 64 & $10^{-4}$ & 40 & $10^{-5}$ \\\hline
AAF          & 100 & 64 & $10^{-4}$ & 40 & $10^{-5}$ \\\hline
EVA          & 100 & 64 & $10^{-4}$ & 40 & $10^{-5}$ \\\hline
WIKI         & 100 & 64 & $10^{-4}$ & 40 & $10^{-5}$ \\\hline
 
\end{tabular}}
\caption{Technical details of the experiments}
\label{tab:technical}
\end{table*}

The Adam optimizer was used in all experiments, with the default $\beta = (0.9, 0.999)$.
The means and standard deviations reported in table~\ref{tab:real_results} are based on 10 repetitions of each experiment, differing in weights initialization and random train-test splits, except for Adience, for which we repeated the experiment five times, using the same train-test splits as the creators of the dataset\footnote{\url{https://github.com/GilLevi/AgeGenderDeepLearning/tree/master/Folds/train_val_txt_files_per_fold}}.

\end{appendix}
\end{document}